\documentclass[journal]{IEEEtran}
\usepackage{times}
\usepackage{epsfig}
\usepackage{cite}
\usepackage{graphicx}
\usepackage{amsmath}
\usepackage{amsthm}
\usepackage{amssymb}
\usepackage{multirow}
\usepackage{booktabs}       
\usepackage{algorithm}
\usepackage{algorithmic}

\usepackage{microtype}
\usepackage{subfigure}
\usepackage{booktabs}
\usepackage{dsfont}
\usepackage{nicefrac}
\usepackage{amsfonts}
\usepackage{stmaryrd}

\usepackage[pagebackref=true,breaklinks=true,letterpaper=true,colorlinks,bookmarks=false]{hyperref}
\begin{document}

\title{Entropy Minimization vs. Diversity Maximization for Domain Adaptation}

\author{Xiaofu~Wu$^\dag$, Suofei~Zhang, Quan~Zhou, Zhen~Yang, Chunming~Zhao and Longin~Jan~Latecki  

\thanks{$^\dag$Corresponding author. This work was supported in part by the National Natural Science Foundation of China under Grants 61372123, 61671253 and by the Scientific Research Foundation of Nanjing University of Posts and Telecommunications under Grant NY213002.}
\thanks{Xiaofu~Wu, Quan~Zhou and Zhen~Yang are with the National Engineering Research Center of Communications and Networking, Nanjing University of Posts and Telecommunications, Nanjing 210003, China (E-mails: xfuwu@ieee.org, \{quan.zhou,yangz\}@njupt.edu.cn).}
\thanks{Suofei Zhang is with the School of Internet of Things, Nanjing University of Posts and Telecommunications, Nanjing 210003, China (E-mail: zhangsuofei@njupt.edu.cn).}
\thanks{Chunming Zhao is with the National Mobile Commun. Research Lab.,
        Southeast University, Nanjing 210096, China (Email: cmzhao@seu.edu.cn).}
\thanks{Longin Jan Latecki is with the Department of Computer and Information
Sciences, Temple University, Philadelphia, Pennsylvania, USA. (Email: latecki@
temple.edu).}}


\maketitle

\begin{abstract}
    Entropy minimization has been widely used in unsupervised domain adaptation (UDA). However, existing works reveal that entropy minimization only may result into collapsed trivial solutions. In this paper, we propose to avoid trivial solutions by further introducing diversity maximization. In order to achieve the possible minimum target risk for UDA, we show that diversity maximization should be elaborately balanced with entropy minimization, the degree of which can be finely controlled with the use of deep embedded validation in an unsupervised manner. The proposed minimal-entropy diversity maximization (MEDM) can be directly implemented by stochastic gradient descent without use of adversarial learning. Empirical evidence demonstrates that MEDM outperforms the state-of-the-art methods on four popular domain adaptation datasets.
\end{abstract}


\newtheorem{plm}{Problem}
\newtheorem{thm}{Theorem}
\newtheorem{conj}{Conjecture}
\section{Introduction}
The recent success of deep learning depends heavily on the large-scale fully-labeled datasets and the development of easily trainable deep neural architectures under the back-propagation algorithm, such as convolutional neural networks (CNNs) and their variants \cite{He2016ResNet,Huang2017DenseNet}. In practical applications, a new target task and its dataset (target domain) may be similar to a known source task and its fully-labeled dataset (source domain). However, the difference between the source and target domains is often not negligible, which makes the previously-trained model not work well for the new task. This is known as domain shift \cite{antonio2011domain-shift}. As the cost of massive labelling is often expensive, it is very attractive for the target task to exploit any existing fully-labeled source dataset and adapt the trained model to the target domain \cite{TIP2020CRTL,TIP2020ADDA,TIP2020Mul,TIP2019LPJT,TIP2018,TIP2018DICD,chang2019bn,roy2019DWT,chen2019PFA,kim2019GPDA}.

This domain adaptation approach is aiming to learn a discriminative classifier in the presence of domain shift \cite{Chen2011CoTrain,bousmalis2016domain}. It can be achieved by optimizing the feature representation to minimize some measures of domain shift, typically defined as the distances between the source and target domain distributions or its degraded form, such as Maximum Mean Discrepancy (MMD) \cite{Long2015MMD,Tzeng2014MMD} or correlation distance \cite{Sun2016ICCV}.

With the invention of generative adversarial networks \cite{Goodfellow2014GAN}, various adversarial methods have been proposed for the purpose of unsupervised domain adaptation \cite{Ganin2015domain,bousmalis2016domain,CyCADA,MCD2018}, where the domain discrepancy distance is believed to be minimized through an adversarial objective with respect to a binary domain discriminator. The domain-invariant features could be extracted whenever this binary domain discriminator cannot distinguish between the distributions of the source and target samples \cite{Ganin2015domain,bousmalis2016domain}.

In recent years,  there is also a broad class of domain adaptation methods, which employ entropy minimization as a proxy for mitigating the the harmful effects of domain shift. The entropy minimization is performed on the target domain, which may take explicit forms \cite{Haeusser2017domain,Tzeng2015domain,Carlucci2015domain,Saito2017domain,Shu2017DIRT-T} or implicit forms \cite{Ganin2015domain,Tzeng2017domain}. Without any further regularization, it may produce trivial solutions \cite{Morerio2018MECA}.

Often, unsupervised domain adaptation (UDA) faces the challenging problem of hyperparameter selection, where the best configuration should be determined without resort to labels in the target dataset. Fortunately, deep embedded validation (DEV) \cite{DEV} tailored to UDA was recently proposed to solve this difficulty, which embeds adapted feature representation in the validation procedure to yield unbiased estimation of the target risk.

\begin{figure*}
\begin{center}
\includegraphics[width=0.65\textwidth]{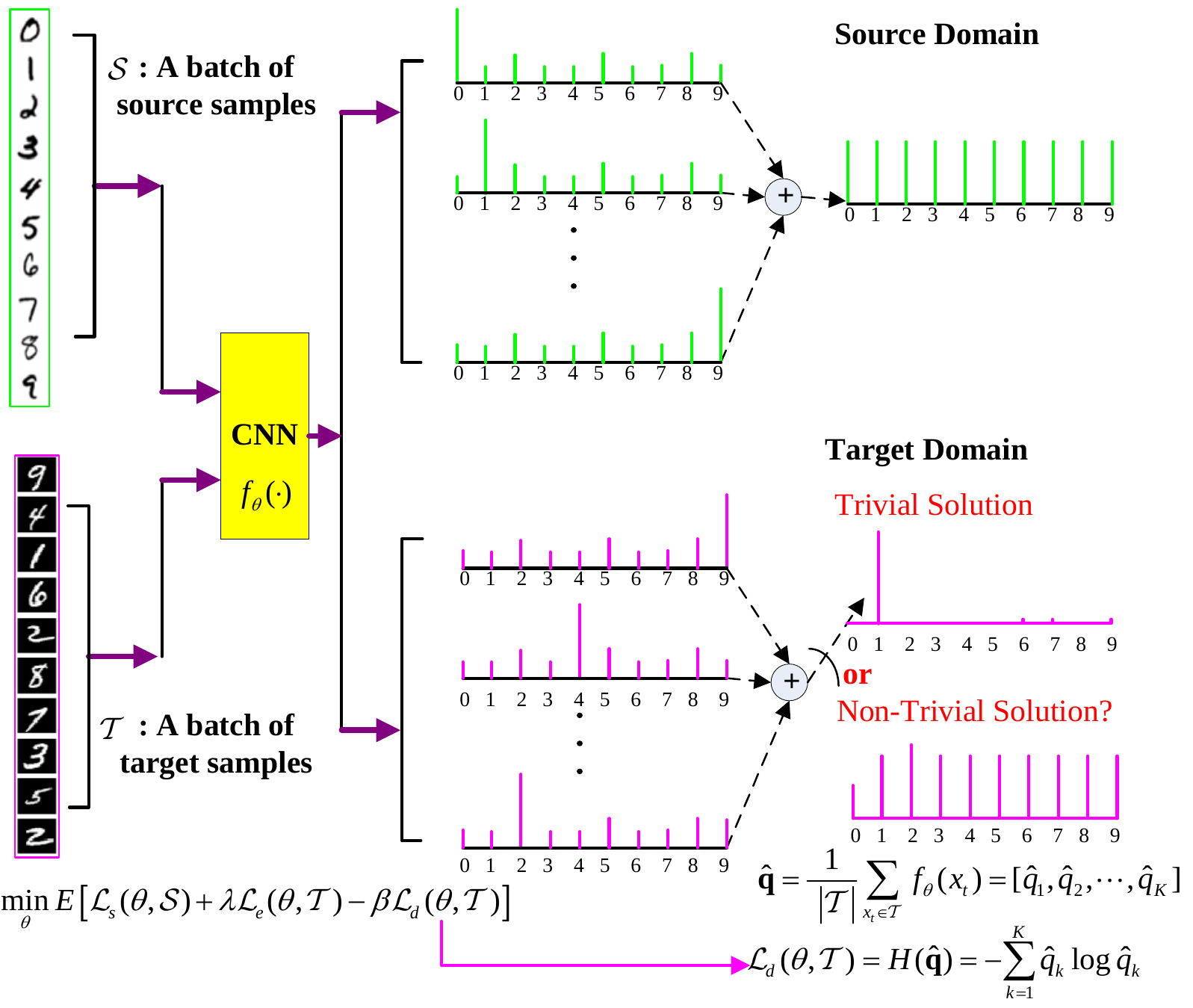} 
\end{center}
   \caption{MEDM tries to maximize the category diversity, which can push entropy minimization away from trivial solutions.}
\label{fig:cpem}
\end{figure*}

In this paper, we make contributions towards close-to-perfect domain adaptation with entropy minimization.
\begin{enumerate}
\item
We propose a minimal-entropy diversity maximization (MEDM) method for UDA. Instead of simply avoiding trivial solutions for entropy minimization, MEDM tries to find a close-to-perfect domain adaptation solution, which achieves the best possible tradeoff between entropy minimization and diversity maximization with the help of DEV \cite{DEV}.

\item
MEDM outperforms state-of-the-art methods on four domain adaptation datasets, including VisDA-2017, ImageCLEF, Office-Home and Office-31. In particular, it boosts a significant accuracy margin on the largest domain adaptation
dataset, VisDA-2017 classification challenge \footnote{All of our experimental results are reproducible and the source codes are available at \href{url}{https://github.com/AI-NERC-NUPT/MEDM}}.
\end{enumerate}

\section{Background and Related Work}
\label{gen_inst}
\subsection{Entropy-Minimization-Only}
Consider the problem of classifying an image $x$ in a $K$-classes problem. For UDA, we are given a source domain $\mathcal{D}_s =\{(x_i^s; y_i^s)\}_{i=1}^{n_s}$ of $n_s$ labeled examples and a target domain $\mathcal{D}_t =\{x_j^t\}_{j=1}^{n_t}$ of $n_t$ unlabeled examples. The source domain and target domain are sampled from joint distributions $P(x^s; y^s)$ and $Q(x^t; y^t)$ respectively, while the identically independently distributed (IID) assumption is often violated as $P\neq Q$. Hence, the problem is to exploit a bunch of labeled images in $\mathcal{D}_s$ for training a statistical classifier that, during inference, provides probabilities for a given test image $x_t \in \mathcal{D}_t$ to belong to each of the $K$ classes. In this paper, we focus on a deep neural-network based classifier $y=f_\theta(x)$ (In general, the classifier network $f_\theta$ depends upon a collection of parameters $\theta$), which provides probabilities of $x$ belonging to each class as
\begin{equation}
f_\theta(x) = \left[\mathbb{P}(y=1|x), \cdots, \mathbb{P}(y=K|x)\right].
\end{equation}

The goal is to design the classifier $y = f_\theta(x)$ such that the target risk $\epsilon_t (f_\theta) = \mathbb{E}_{(x^t; y^t)\sim Q} [f_\theta(x^t) \neq y^t]$ can be minimized. Since the target risk cannot be computed in the scenario of UDA, the domain adaptation theory \cite{BenDavid2007Bound} \cite{Blitzer2008Bound} suggests to bound the target risk with the sum of the cross-domain discrepancy $D(P;Q)$ and the source risk $\epsilon_s (f_\theta) = \mathbb{E}_{(x^s; y^s)\sim P} [f_\theta(x^s) \neq y^s]$. By jointly minimizing the source risk and the cross-domain discrepancy $D(P;Q)$, various domain adaptation methods were extensively proposed, which differ mainly in the choice of $D(P;Q)$.

For supervised learning on the source domain, the classifier is trained to minimize the standard supervised loss
\begin{eqnarray}
\mathcal{L}_s(\theta, D_s) = \frac{1}{|\mathcal{D}_s|} \sum_{(x, y)\in \mathcal{D}_s} \ell(y, f_\theta(x))
\end{eqnarray}
with $\ell(y, \hat{y})=\langle y, \hat{y} \rangle=-\sum_{j=1}^K y_j \log \hat{y}_j$ and $|\mathcal{D}_s|$ denotes the cardinality of the set $\mathcal{D}_s$.

To adapt to the unlabeled target domain, a large class of domain adaptation methods also minimize
the entropy loss on the target domain
\begin{equation}
\label{e-loss}
\mathcal{L}_e(\theta, \mathcal{D}_t) = -  \frac{1}{|\mathcal{D}_t|}  \sum_{x_t \in \mathcal{D}_t} \langle f_\theta(x_t), \log f_\theta(x_t) \rangle
\end{equation}
as an efficient regularization technique. Therefore, the standard domain adaptation method with entropy-minimization-only (EMO) seeks to solve the following problem
\begin{equation}
\label{entropy-min}
\min_{\theta} \left[\mathcal{L}_s(\theta, D_s) + \lambda \mathcal{L}_e(\theta,\mathcal{D}_t)\right], \lambda > 0.
\end{equation}

The EMO presented in (\ref{entropy-min}) was first proposed in \cite{EM_Bengio2005} for semi-supervised learning, where a decision rule is to be learned from labeled and unlabeled data, and EMO (\ref{entropy-min}) enables to incorporate unlabeled data in the standard supervised learning. For the scenario of UDA considered in this paper, the difference is that unlabeled samples are sampled from the target-domain distribution $Q$, which may differ considerably from the source-domain distribution $P$.

\subsection{Insufficiency of EMO for UDA}
Note that the minimization of the target risk $\epsilon_t (f_\theta)$ could push the network prediction $f_\theta(x_t)$ towards the true solution $y_t=[y_1,\cdots,y_K]$ with $y_k \in \{0,1\}, \sum_k y_k=1$, namely, $f_\theta(x_t) \rightarrow [0, \cdots, 1, \cdots, 0]$, which results into the minimum value of entropy (zero).  This means that entropy minimization is a necessary condition for the minimization of the target risk $\epsilon_t (f_\theta)$. Hence, \textit{entropy minimization may be more direct and simpler for end-to-end training of $\theta$ in order to minimize the target risk, compared to the use of more complicated cross-domain discrepancy}.  Unfortunately, as a necessary but not sufficient condition for minimization of the target risk $\epsilon_t(f_\theta)$ \cite{Morerio2018MECA}, this simple technique may result into trivial solutions, as demonstrated in the Appendix.

\begin{plm}
As entropy minimization is necessary but not sufficient for minimization of the target risk,  it is natural to ask if we can pose some further regularization to push the optimizer to find the global minima instead of the trivial local minima.
\end{plm}

\subsection{Existing UDA Approaches with EM}
Entropy minimization was first proposed in \cite{EM_Bengio2005} for semi-supervised learning. In many UDA scenarios with very limited domain-shift between source and target domains, it does work, as demonstrated later in experiments. When the effect of domain-shift increases, the optimization of (\ref{entropy-min}) is not enough for the purpose of UDA \cite{Morerio2018MECA}. Hence, various ancillary adaptation techniques were invoked, such as covariance alignment \cite{Morerio2018MECA}, batch normalization \cite{Carlucci2015domain} or learning by association \cite{Haeusser2017domain}.

It was argued in \cite{Morerio2018MECA} that entropy minimization could be achieved by the optimal alignment of second order statistics between source and target domains and therefore a hyper-parameter validation method was proposed for balancing  the reduction of the domain shift and the supervised classification on the source domain in an optimal way.

In \cite{Carlucci2015domain}, a novel domain alignment layer was introduced for reducing the domain shift by matching source and target distributions to a reference one and  entropy minimization was also explicitly employed, which was believed to promote classification models with high confidence on unlabeled samples.

Long et al. \cite{long2016RTN} used entropy minimization in their approach to directly measure how far samples are from a decision boundary by calculating entropy of the classifier¡¯s output.

In the appendix of \cite{Saito2018drop}, Satio et al. proposed an entropy-based adversarial dropout regularization approach, which employed the entropy of the target samples in implementing min-max adversarial training.

In \cite{CDAN}, entropy conditioning was employed that controls the uncertainty of classifier
predictions to guarantee the transferability, which can help the proposed conditional adversarial domain adaptation (CDAN) to converge to better solutions.

\section{Minimal-Entropy Diversity Maximization}
\label{headings}
\subsection{Proposed Method}
As the training of network is often implemented over batches of samples, the supervised loss for a given source batch $\mathcal{S}$ (for example, $|\mathcal{S}|=32$ for the batch size of 32) is accordingly modified as
\begin{eqnarray}
\mathcal{L}_s(\theta,\mathcal{S}) = \frac{1}{|\mathcal{S}|} \sum_{(x, y)\in \mathcal{S}} \ell(y, f_\theta(x)).
\end{eqnarray}

As shown in Figure \ref{fig:cpem}, domain adaptation requires some regularization techniques for pushing the network towards correct class prediction of unlabeled target samples. For ease of implementation, the regularization is often performed over batches of target samples. Let $\mathcal{T} \subset \mathcal{D}_t$ be any random batch of samples from target domain, which has the same batch size as that of $\mathcal{S}$, namely,  $|\mathcal{T}| = |\mathcal{S}|$.

With each unlabeled image $x_t \in \mathcal{T}$ as input, we can perform inference over the network $f_\theta$ to obtain the softmax predictions $f_\theta(x_t)$. Then, we can compute the predicted category distribution in $\mathcal{T}$ as
\begin{eqnarray}
\label{cat_prob}
\hat{\mathbf{q}}(\mathcal{T}) = \frac{1}{|\mathcal{T}|} \sum_{x_t \in \mathcal{T}} f_\theta(x_t) \triangleq [\hat{q}_1, \hat{q}_2, \cdots, \hat{q}_K],
\end{eqnarray}
where
\begin{eqnarray*}
\hat{q}_k = \frac{1}{|\mathcal{T}|} \sum_{x_t \in \mathcal{T}}\mathbb{P}(y_t=k|x_t),
\end{eqnarray*}
and $\sum_{k=1}^K \hat{q}_k  = 1$. \textit{Note that $\hat{\mathbf{q}}(\mathcal{T})$ is computed over $\mathcal{T}$, which is dynamically changed during the batch-based training}.

Without any labelling information available, the mean entropy loss over $\mathcal{T}$ can be computed as
\begin{equation}
\label{e-loss}
\mathcal{L}_e(\theta,\mathcal{T})=  - \frac{1}{|\mathcal{T}|} \sum_{x_t \in \mathcal{T}} \langle f_\theta(x_t), \log f_\theta(x_t)\rangle.
\end{equation}

The use of EMO may produce trivial solutions as shown in Figure \ref{fig:cpem}.  By noting that a trivial solution shown in Figure \ref{fig:cpem} often has just one category, a nontrivial domain adaptation method may resort to producing sufficient category diversity in its solution.

In this paper, we employ the entropy of $\hat{\mathbf{q}}(\mathcal{T})=[\hat{q}_1, \hat{q}_2, \cdots, \hat{q}_K]$ (\ref{cat_prob}) for measuring the category diversity in a given target batch $\mathcal{T}$. Formally, this category diversity over $\mathcal{T}$ can be measured as
\begin{equation}
\label{div-loss}
\mathcal{L}_d(\theta, \mathcal{T}) \triangleq H\left(\mathbf{\hat{q}}(\mathcal{T})\right) = -\sum_{k=1}^K \hat{q}_k \log \hat{q}_k.
\end{equation}
As this diversity metric does not require any priori information about the true category distribution $\mathbf{q}$ over $\mathcal{D}_t$, its computation is easy to implement in practice. Note that random shuffling should be employed in training for maximizing (\ref{div-loss}).

The objective of the proposed MEDM is to
\begin{eqnarray}
\label{total-loss}
&&\min_{\theta} E_{\mathcal{S}, \mathcal{T}} \left[ \mathcal{L}_s(\theta, \mathcal{S}) + \lambda \mathcal{L}_e(\theta, \mathcal{T}) - \beta \mathcal{L}_d(\theta, \mathcal{T})\right]
\end{eqnarray}
where $E[\cdot]$ denotes the expectation and $\lambda, \beta \ge 0$ are weighting factors.  Given $\lambda,\beta$, this involves the optimization of $\theta$ for the minimization of a single total loss (\ref{total-loss}), which can be directly implemented by stochastic gradient descent without use of adversarial learning.

\subsection{Entropy-Minimization vs. Diversity-Maximization}
As shown in (\ref{total-loss}), our proposed MEDM may encourage to make prediction evenly across the batch, which, however, does not necessarily produce the evenly-distributed categories. Let $\mathbf{q} =[q_1,\cdots, q_K]$  be the true category distribution of the target dataset, where $q_k$ denotes the proportion of samples of the $k$-th class among all target samples.

\begin{thm}
Consider the EMO method ($\beta=0$) in (\ref{total-loss}). If there exists a solution $\theta^*$ of (\ref{total-loss}) with optimal entropy minimization, we have that
\begin{eqnarray*}
E_{\mathcal{T}}\left[\mathcal{L}_d(\theta^*, \mathcal{T})\right] = H(\mathbf{q^*}),
\end{eqnarray*}
where $\mathbf{q}^*=[q_1^*,\cdots, q_K^*]$ is the inferred category distribution of the target dataset when inferring over the network $\theta^*$.
\end{thm}
\begin{proof}
In the case of $\beta=0$, diversity maximization is not included in (\ref{total-loss}). With optimal entropy minimization, it means that
\begin{eqnarray*}
\mathcal{L}_e(\theta^*, \mathcal{T}) = 0.
\end{eqnarray*}
Since the entropy is always non-negative, we have that
\begin{eqnarray*}
-\langle f_\theta(x_t), \log f_\theta(x_t)\rangle = 0, \forall x_t \in \mathcal{T}.
\end{eqnarray*}
Hence, the network prediction $f_\theta(x_t)$ for any $x_t \in \mathcal{T}$ should present a peaky form, namely, $f_\theta(x_t) \rightarrow [0, \cdots, 1, \cdots, 0]$.

Given a random batch of samples $\mathcal{T}$ inputting to the network $\theta^*$, we have that
\begin{eqnarray*}
E_{\mathcal{T}}\{\hat{q}_k\} =\frac{1}{|\mathcal{T}|} E_{\mathcal{T}} \left\{\sum_{x_t \in \mathcal{T}} f_{\theta^*}^k (x_t)\right\} = \frac{1}{|\mathcal{T}|} \cdot (q_k^* |\mathcal{T}|) = q_k^*.
\end{eqnarray*}
Therefore,
\begin{eqnarray*}
E_{\mathcal{T}}\left\{\mathcal{L}_d(\theta^*, \mathcal{T})\right\} = E_{\mathcal{T}} \{H(\mathbf{\hat{q}})\}=  H(\mathbf{q^*}).
\end{eqnarray*}
\end{proof}

Without the use of diversity maximization, EMO often results into trivial solutions, namely, $\max q^*_k \gg  1- \max_k q^*_k$, where the predicted single-class samples may dominate among others. With the use of diversity maximization, it may encourage to make prediction evenly across the batch, since the maximum value of $\mathcal{L}_d(\theta^*, \mathcal{T})$ could be achieved whenever $\mathbf{q^*}=[1/K, \cdots,1/K]$. Actually, there exists a tradeoff by adjusting the parameters $\lambda,\beta$ as justified in what follows.

Assume that the network $f_\theta(x)$ could be decomposed into two subnetworks, namely, $f_\theta(x) = C(F(x))$ , where $F$ denotes a feature extraction subnetwork and $C$ denotes a classifier over the feature space $\mathcal{F}$. For domain-adaptation, it is often assumed that the conditional distributions are unchanged by $F$, i.e., $P(y|F(x)) = Q(y|F(x))$.

\begin{algorithm}
    \caption{Fast Model Selection Process in MEDM}
    \label{simProc}
    \begin{algorithmic}[1]
    \REQUIRE  $\mathcal{D}_s=\mathcal{D}_{\text{train}} \cup \mathcal{D}_{\text{val}}, \mathcal{D}_t$; \\
    $\lambda_1^L=\{\lambda_l\}_{i=1}^L, \beta_1^B=\{\beta_l\}_{l=1}^B$
        \STATE {Training Initialization: $\beta \leftarrow 0$.}

        \medskip

        \FOR {$\lambda \leftarrow$ $\lambda_1, \cdots, \lambda_L$}
            \STATE {Training the network $\theta_i$ over $\mathcal{D}_{\text{train}}$ and $\mathcal{D}_t$}:
\begin{eqnarray*}
\theta_i = \arg \min_{\theta} E \left[ \mathcal{L}_s(\theta,\mathcal{S}) + \lambda \mathcal{L}_e(\theta,\mathcal{T})\right]
\end{eqnarray*}
            \STATE {If $\mathcal{L}_e(\theta_i, \mathcal{T}) \rightarrow 0$}: $\lambda^*=\lambda$ and break
        \ENDFOR

        \medskip

        \FOR {$\beta \leftarrow$ $\beta_1, \cdots, \beta_B$}
            \STATE {Training the network $\theta_l$ over $\mathcal{D}_{\text{train}}$ and $\mathcal{D}_t$}:
\begin{eqnarray*}
\theta_l = \arg\min_{\theta} E \left[ \mathcal{L}_s(\theta,\mathcal{S}) + \lambda^* \mathcal{L}_e(\theta, \mathcal{T}) - \beta \mathcal{L}_d(\theta, \mathcal{T})\right]
\end{eqnarray*}
        \ENDFOR
        \STATE{\bf Deep Embedded Validation \cite{DEV}:}
        \begin{enumerate}
        \item
        Get DEV Risks of all models $\mathcal{R} = \{\text{GetRisk}(\theta_l)\}_{l=1}^B$ over $\mathcal{D}_{\text{val}}$
        \item
        Rank the best model $l^* = \arg \min_{1\le l \le B} \mathcal{R}_l$
        \end{enumerate}
    \end{algorithmic}
\end{algorithm}

Let $\mathbf{X}_s = \{x_s\}_{s=1}^{n_s}$  and $\mathbf{X}_t = \{x_t\}_{t=1}^{n_t}$. Then, whenever $x_t \in F^{-1}\left(F(\mathbf{X}_s) \cap F(\mathbf{X}_t)\right)$, one can expect that the network can give a correct prediction with minimal entropy. While $x_t \in F^{-1}\left(F(\mathbf{X}_t) - F(\mathbf{X}_s)\right)$, it may encourage to make prediction towards a single class with entropy minimization, since there are simply no other constraints to be enforced.
Hence, we have the following conjecture.
\begin{conj} Consider the EMO method ($\beta=0$) in (\ref{total-loss}).  If there exists a solution $\theta^*$ of (\ref{total-loss}) with optimal entropy minimization, we have that
\begin{eqnarray*}
E_{\mathcal{T}}\left[\mathcal{L}_d(\theta^*, \mathcal{T})\right] \le H(\mathbf{q}),
\end{eqnarray*}
where $\mathbf{q}=[q_1,\cdots, q_K]$ is the ground-truth category distribution of the target dataset.
\end{conj}

Let $f_{\theta^*}$ be the perfect domain-adaptation classier, which minimizes combined source and target risk \cite{Blitzer2008Bound}:
\begin{equation}
\label{eq:ideal}
f_{\theta^*}=\arg \min_{f_\theta \in \mathcal{H}} \epsilon_s(f_{\theta}) + \epsilon_t(f_{\theta}).
\end{equation}
where $\mathcal{H}$ is the space of classifiers. Therefore, when inferring target samples over $f_{\theta^*}$, one can expect reliable predictions or \textit{predictions with low entropies}. Hence, Theorem 1 may still hold in this case. By restricting $\mathcal{H}$ to be the solution space of (\ref{total-loss}) with $\lambda,\beta \ge 0$, we expect that the same conclusion holds. In fact, \textit{extensive experiments show that an explicit inclusion of entropy minimization in (\ref{total-loss}) can easily drive the trained CNN model towards small predictive entropy, even coexisting with diversity maximization}. Therefore, we believe that the perfect domain-adaptation classifier under the framework of (\ref{total-loss}) may output predictions with low entropies, which means that Theorem 1 may still hold, as shown in what follows.

\begin{conj} Consider the perfect domain-adaptation solution of $f_{\theta^*}^{(\lambda^*,\beta^*)}$ under the framework of (\ref{total-loss}), namely,
\begin{eqnarray}
f_{\theta^*}^{(\lambda^*,\beta^*)}=\arg \min_{\theta,\lambda,\beta} \epsilon_s(f_{\theta}^{(\lambda,\beta)}) + \epsilon_t(f_{\theta}^{(\lambda,\beta)}),
\end{eqnarray}
we have that
\begin{eqnarray*}
E_{\mathcal{T}}\left[\mathcal{L}_d(\theta^*, \mathcal{T})\right] \approx H(\mathbf{q}^*),
\end{eqnarray*}
where $\mathbf{q}^*=[q_1^*,\cdots, q_K^*]$ is the inferred category distribution of the target dataset over the network $\theta^*$.
\end{conj}

Since the target risk for $f_{\theta^*}^{(\lambda^*,\beta^*)}$ is expected to be small, we have that $\mathbf{q}^* \rightarrow \mathbf{q}$ and $H(\mathbf{q}^*)\rightarrow H(\mathbf{q})$.  With \textit{random shuffling} for batch-based training, $\mathcal{L}_d(\theta^*, \mathcal{T}) \rightarrow H(\mathbf{q}^*)$ holds with a high probability whenever the training process for (\ref{total-loss}) under the setting of $(\lambda^*,\beta^*)$ converges. This may partially support the reasonability of the use of (\ref{total-loss}).

Although we do not know the perfect domain-adaptation solution of $f_{\theta^*}^{(\lambda^*,\beta^*)}$, one can search over the space of $(\lambda,\beta)$ and further resort to the validation technique \cite{DEV}. The essential process can be stated as follows.  When $\beta=0$ in (\ref{total-loss}), it often results in the trivial solution of one domain single class, which means that $\mathcal{L}_d(\theta^*, \mathcal{T}) \approx 0$. When $\beta$ increases from 0, we would expect that the category diversity (\ref{div-loss}) increases correspondingly, which can help to avoid the trivial solutions. The problem now is how to determine the best value of $\lambda,\beta$ in (\ref{total-loss}) for finding a close-to-perfect domain adaptation solution. Fortunately, this was recently investigated in \cite{DEV}.
\begin{table}
\caption{Accuracy (\%) of ResNet-50 model on VisDA-2017}
\begin{center}
\label{my-label}
\begin{tabular}{l||c}
\toprule[1.5pt]
Method     & Synthetic $\rightarrow$ Real \\\hline
GTA~\cite{GTK}      & 69.5 \\
MCD~\cite{MCD2018} & {69.8}\\
CDAN~\cite{CDAN} & {70.0}\\
MDD ~\cite{MDD} & {74.6}\\\hline
MEDM(ours) & {\bf 79.6}\\

\bottomrule[1.5pt]
\end{tabular}
\end{center}
\label{tb:visda50}
\end{table}
\subsection{Model Selection via Deep Embedded Validation}
For the proposed MEDM, the selection of hyperparameters ($\lambda,\beta$) is of great importance for the final performance. For UDA, the model selection should be decided without access to the labels in the target dataset.
Fortunately, the recently-proposed deep embedded validation \cite{DEV} has been proved very efficient for model selection,  which embeds adapted feature representation into the validation procedure to obtain unbiased estimation
of the target risk with bounded variance.

Consider that the feature extractor $F$ is an end-to-end training solution of (\ref{total-loss}), it is closely connected to the parameters $\lambda,\beta$ in (\ref{total-loss}), i.e., $F\triangleq F_{\lambda,\beta}$. Let $\lambda_1^L=\{\lambda_l\}_{i=1}^L$ be a finite collection of $\lambda_l$, where $\lambda_1 < \lambda_2 < \cdots < \lambda_L$.
Assume that $\theta^*$ is the optimal hyper-parameter which attains the optimum of (\ref{total-loss}) with the smallest possible value of $\lambda^*$ among $\lambda_1^L$, which implies
\begin{eqnarray*}
E_{\mathcal{S}, \mathcal{T}} \left[ \mathcal{L}_s(\theta^*, \mathcal{S}) + \lambda^* \mathcal{L}_e(\theta^*, \mathcal{T})\right] = \min(\lambda^*), \\
\min(\lambda^*) > \min(\lambda), \forall \lambda \in \lambda_1^L, \lambda > \lambda^*.
\end{eqnarray*}
This is because that $\mathcal{L}_e(\theta^*, \mathcal{T}) \approx 0$ could be more easily achieved whenever $\lambda \ge \lambda^*$. With the smallest possible value of $\lambda$, the degree of entropy minimization can be controlled during training with better discriminability while maintaining the transferability.

Therefore, we propose a fast model selection process for MEDM as shown in Algorithm \ref{simProc}, which can reduce the search space for two hyperparameters $\lambda, \beta$. Practically, we use $\mathcal{L}_e(\theta_i, \mathcal{T}) \le \epsilon$ for deciding if $\mathcal{L}_e(\theta_i, \mathcal{T}) \rightarrow 0$ and $\epsilon = 0.2$ is employed in experiments.

\begin{table*}[t]
\caption{Accuracy (\%) of ResNet-101 model on the VisDA dataset}
\begin{center}
\label{my-label}
\begin{tabular}{l||cccccccccccc|c}
\toprule[1.5pt]
Method     & plane & bcycl & bus  & car  & horse & knife & mcycl & person & plant & sktbrd & train & truck & mean \\\hline
DAN~\cite{Long2015MMD}     & 87.1      & 63.0      & 76.5 & 42.0   & {90.3}  & 42.9  & {85.9}        & 53.1   & 49.7  & 36.3       & {85.8}  & 20.7  & 61.1 \\
RevGrad~\cite{Ganin2015domain}      & 81.9      & { 77.7}    & 82.8 & 44.3 & 81.2  & 29.5  & 65.1        & 28.6   & 51.9  & { 54.6}       & 82.8  & 7.8   & 57.4 \\
MCD~\cite{MCD2018} & 87.0&60.9&{ 83.7}& {64.0}&88.9&{79.6}&84.7&{76.9}&{ 88.6}&40.3&83.0&25.8&{71.9}\\
CDAN~\cite{CDAN} & 85.2 &66.9 &{ 83.0}& {50.8}&84.2&{74.9}&88.1&{74.5}&{ 83.4}&76.0&81.9&38.0&{73.7}\\
BSP+CDAN~\cite{BSP} & 92.4&61.0&{ 81.0}& {57.5}&89.0&{80.6}&90.1&{77.0}&{ 84.2}&77.9&82.1&38.4&{75.9}\\\hline
MEDM(ours) & {\bf93.5}& \bf 80.4& \bf 90.8 & {\bf 70.3} & \bf 92.8 &{\bf 87.9}&{\bf 91.1}&{\bf 79.8}&{\bf 93.7}&{\bf 83.6}& \bf 86.1&{\bf 38.7} &{\bf 82.4}\\

\bottomrule[1.5pt]
\end{tabular}
\end{center}
\label{tb:visda101}
\end{table*}

\begin{table*}[t]
\caption{Accuracy (\%) on ImageCLEF-DA dataset for unsupervised domain adaptation with ResNet-50}
\begin{center}
\begin{tabular}{l||cccccc|c}
\toprule[1.5pt]
Method& I $\rightarrow$ P & P $\rightarrow$ I & I $\rightarrow$ C & C $\rightarrow$ I & C $\rightarrow$ P & P $\rightarrow$ C  & Avg \\
\midrule
 DAN~\cite{Long2015MMD} & $75.0\pm0.4$ & $86.2\pm0.2$ & $93.3\pm0.2$ & $84.1\pm0.4$ & $69.8\pm0.4$ & $91.3\pm0.4$ & 83.3\\
 RTN~\cite{long2016RTN} & $75.6\pm0.3$ & $86.8\pm0.1$ & $95.3\pm0.1$ & $86.9\pm0.3$ & $72.7\pm0.3$ & $92.2\pm0.4$ & 84.9\\
RevGrad~\cite{Ganin2015domain}   & $75.0\pm0.6$ & $86.0\pm0.3$ & $96.2\pm0.4$ & $87.0\pm0.5$ & $74.3\pm0.5$ & $91.5\pm0.6$ & 85.0\\
 MADA~\cite{Long2018AAAI}  & $75.0\pm0.3$ & $87.9\pm0.2$ & $96.0\pm0.3$ & $88.8\pm0.3$ & $75.2\pm0.2$ & $92.2\pm0.3$ & 85.8\\
 CDAN~\cite{CDAN}  & $77.7\pm0.3$ & $90.7\pm0.2$ & $\textbf{97.7}\pm0.3$ & $91.3\pm0.3$ & $74.2\pm0.2$ & $94.3\pm0.3$ & 87.7\\
 \midrule
 MEDM(Ours) &$\textbf{78.5}\pm0.5$ & $\textbf{93.0}\pm0.5$ & $96.1\pm0.2$ & $\textbf{92.8}\pm0.5$ & $\textbf{77.2}\pm0.7$ & $\textbf{95.5}\pm0.4$ & \textbf{88.9}\\
 \bottomrule[1.5pt]
\end{tabular}
\label{table:clef}
\end{center}
\end{table*}

\begin{table*}[htbp]
	\caption{Accuracy (\%) on {Office-Home} for unsupervised domain adaptation with ResNet-50}
	\addtolength{\tabcolsep}{-5pt}
	\centering
	\vskip 0.05in
	\resizebox{0.75\textwidth}{!}{%
		\begin{tabular}{l||cccccccccccc|c}
			\toprule
			Method                          & Ar$\shortrightarrow$Cl & Ar$\shortrightarrow$Pr & Ar$\shortrightarrow$Rw & Cl$\shortrightarrow$Ar & Cl$\shortrightarrow$Pr & Cl$\shortrightarrow$Rw & Pr$\shortrightarrow$Ar & Pr$\shortrightarrow$Cl & Pr$\shortrightarrow$Rw & Rw$\shortrightarrow$Ar & Rw$\shortrightarrow$Cl & Rw$\shortrightarrow$Pr & Avg           \\
			\midrule
DAN~\cite{Long2015MMD}        & 43.6                   & 57.0                   & 67.9                   & 45.8                   & 56.5                   & 60.4                   & 44.0                   & 43.6                   & 67.7                   & 63.1                   & 51.5                   & 74.3                   & 56.3          \\
DANN~\cite{DANN}   & 45.6                   & 59.3                   & 70.1                   & 47.0                   & 58.5                   & 60.9                   & 46.1                   & 43.7                   & 68.5                   & 63.2                   & 51.8                   & 76.8                   & 57.6          \\
JAN~\cite{JAN}        & 45.9                   & 61.2                   & 68.9                   & 50.4                   & 59.7                   & 61.0                   & 45.8                   & 43.4                   & 70.3                   & 63.9                   & 52.4                   & 76.8                   & 58.3          \\
CDAN~\cite{CDAN}      & 50.7                   & 70.6                   & 76.0                   & 57.6                   & 70.0                   & 70.0                   & 57.4                   & 50.9                   & 77.3                   & 70.9                   & 56.7                   & 81.6                   & 65.8          \\
LPJT~\cite{TIP2019LPJT}  & 32.5                   & 54.8                   & 57.1                 & 34.4                   & 53.8                   & 53.0                   & 35.6                   & 35.3                 & 60.9                   & 45.6                   & 39.4                   & 67.8                   & 47.5          \\
MDD~\cite{MDD} & {54.9} & {73.7} & {77.8}  & {60.0} & {71.4}  & {71.8} & {61.2} & \textbf{53.6} & {78.1} & \textbf{72.5} & \textbf{60.2} & {82.3}  & {68.1} \\
 \midrule
MEDM(Ours) & \textbf{57.1} & \textbf{76.1} & \textbf{80.0}  & \textbf{62.0} & \textbf{72.7}  & \textbf{76.0} & \textbf{62.3} & {53.4} & \textbf{81.2} & \textbf{69.9} & {59.8} & \textbf{83.9}  & \textbf{69.5} \\
			\bottomrule[1.5pt]
		\end{tabular}%
	}
	\label{table:office-home}
\end{table*}

\begin{table*}[t]
\caption{Accuracy (\%) on Office-31 dataset for unsupervised domain adaptation with ResNet-50}
\begin{center}
\begin{tabular}{l||cccccc|c}
\toprule
Method& A $\rightarrow$ W & D $\rightarrow$ W & W $\rightarrow$ D & A $\rightarrow$ D & D $\rightarrow$ A & W $\rightarrow$ A  & Avg \\
\midrule
 DAN~\cite{Long2015MMD}  & $83.8\pm0.4$ & $96.8\pm0.2$ & $99.5\pm0.1$ & $78.4\pm0.2$ & $66.7\pm0.3$ & $62.7\pm0.2$ & 81.3\\
RevGrad~\cite{Ganin2015domain}  & $82.0\pm0.4$ & $96.9\pm0.2$ & $99.1\pm0.1$ & $79.7\pm0.4$ & $68.2\pm0.4$ & $67.4\pm0.5$ & 82.2\\
 MADA~\cite{Long2018AAAI}  & $90.0\pm0.1$ & $97.4\pm0.1$ & $99.6\pm0.1$ & $87.8\pm0.2$ & $70.3\pm0.3$ & $66.4\pm0.3$ & 85.2\\
 CDAN~\cite{CDAN}  & $94.1\pm0.1$ & $98.6\pm0.1$ & $100.0\pm0.0$ & $92.9\pm0.2$ & $71.0\pm0.3$ & $69.3\pm0.3$ & 87.7\\
 CRTL~\cite{TIP2020CRTL}  & $77.4$ & $95.7$ & $97.6$ & $79.5$ & $\textbf{81.9}$ & $\textbf{81.8}$ & 85.6\\
BSP+CDAN~\cite{BSP}  & $93.3\pm0.2$ & $98.2\pm0.2$ & $100.0\pm0.0$ & $93.0\pm0.2$ & $73.6\pm0.3$ & $72.6\pm0.3$ & 88.5\\
 MDD~\cite{MDD}  &$\textbf{94.5}\pm0.3$ & $98.4\pm0.1$ & $100.0\pm0.0$ & $\textbf{93.5}\pm0.2$ & $74.6\pm0.3$ & $72.2\pm0.1$ & 88.9\\
 \midrule
 MEDM(Ours) &$93.4\pm0.6$ & $\textbf{98.8}\pm0.1$ & $\textbf{100.0}\pm0.0$ & $93.4\pm0.5$ & $74.2\pm0.2$ & $75.4\pm0.4$ & \bf89.2\\

\bottomrule[1.5pt]
\end{tabular}
\label{table:office31}
\end{center}
\end{table*}

\section{Experiments}
We evaluate MEDM with state-of-the-art domain adaptation methods for various transferring tasks, which include VisDA-2017, ImageCLEF-DA,  Office-Home, and Office-31 datasets.  VisDA-2017 is  known as the largest and highly unbalanced DA dataset and ImageCLEF-DA is a small but balanced dataset, while both Office-Home and Office-31 are slightly unbalanced DA datasets with a large number of classes (65 for Office-Home and 31 for Office-31).

\subsection{Experimental Setting}
 Throughout the experiments, we employ deep neural network architecture detailed as follows.  It has a pre-trained ResNet-50/101, followed by two fully-connected layers, FC-1 of size $2048\times1024$ and FC-2 of size $1024\times K$. Batch-normalization, ReLU activation and dropout are only employed at the FC-1 layer. The dropout rate is set to 0.5. The last label prediction layer of the network is omitted and features are extracted from the second to last layer. The Adam optimizer is employed with a learning rate of 0.0001. The batch size is set to 32. The learning rates of the layers trained from scratch are set to be 100 times those of fine-tuned layers. For model selection, we assume that $\lambda,\beta \in \{0.0, 0.1,0.2,0.3,0.4,0.5,0.6,0.7,0.8,0.9.1.0\}$. Finer search of $\lambda,\beta$ may lead to the possibly-better performance.

We report the test accuracy results of MEDM, which are compared with state-of-the-art methods: Deep Adaptation Network (DAN) \cite{Long2015MMD}, Reverse Gradient (RevGrad) \cite{Ganin2015domain}, Domain Adversarial Neural Network (DANN) \cite{DANN}, Residual Transfer Network (RTN) \cite{long2016RTN},  Multi-Adversarial Domain Adaptation (MADA) \cite{Long2018AAAI}, Generate to Adapt (GTA) \cite{GTK}, Maximum-Classifier-Discrepancy (MCD)\cite{MCD2018}, Conditional Domain Adversarial Network (CDAN) \cite{CDAN}, Locality Preserving Joint Transfer (LPJT) \cite{TIP2019LPJT}, Class-specific Reconstruction Transfer Learning (CRTL) \cite{TIP2020CRTL}, Margin Disparity Discrepancy (MDD) \cite{MDD}, Batch Spectral Penalization (BSP) + CDAN \cite{BSP}.

\begin{table*}[t]
\caption{The effect of $\beta$ on diversity maximization over VisDA-2017 (Ground-truth category diversity $H(\mathbf{q})=2.3927$)}
\begin{center}
\label{my-label}
\begin{tabular}{l||cccccccccccc|c|c}
\toprule[1.5pt]
$\beta$     & plane & bcycl & bus  & car  & horse & knife & mcycl & person & plant & sktbrd & train & truck & \bf mean &\bf diversity \\\hline
$0.2$ & {91.6}&77.6& 66.4 & {\bf 1.6} & 90.7 &{ 83.4}&{ 89.4}&{ 78.6}&{89.7}&{74.9}&89.6&{\bf 0.3} &{69.5} &{\bf 2.0619}\\
$0.3$ & {94.4}&76.2& 87.3 & {61.1} & 91.1 &{ 79.6}&{ 88.0}&{ 80.0}&{92.3}&{78.9}&89.7&{35.0} &{79.4} &{\bf 2.2446}\\
$0.4$ & {92.7}&83.1& 82.2 & {65.8} & 89.2 &{ 89.9}&{ 79.8}&{ 78.6}&{91.3}&{77.7}&90.9&{33.8} &{79.6} &{\bf 2.2493}\\
$0.5$ & {94.2}&77.8& 80.9 & {58.4} & 90.9 &{\bf 25.9}&{ 81.4}&{ 76.1}&{89.1}&{67.6}&89.7&{40.1} &{72.7} &{\bf 2.2704}\\
\bottomrule[1.5pt]
\end{tabular}
\end{center}
\label{tb:visda50-beta}
\end{table*}

\subsection{VisDA-2017}
 The Visual Domain Adaption (VisDA) challenge \cite{VisDA} aims to test domain adaptation methods's ability to transfer source knowledge and adapt it to novel target domains. As the largest domain-adaptation dataset, the VisDA dataset contains 280K images across 12 categories from the training, validation, and testing domains. The training domain (the source domain) is a set of synthetic 2D renderings of 3D models generated from different angles and with different lighting conditions, while the validation domain (the target domain) is a set of realistic photos. The source domain contains 152,397 synthetic images, and the target domain has 55,388 real images.

 \textit{Note that the target domain is highly-unbalanced, where the number of samples for each category is $[l_1,\cdots,l_{12}]=$[3646, 3475, 4690, 10401, 4691, 2075, 5796, 4000, 4549, 2281, 4236, 5548].  Therefore, the VisDA-2017 also serves to justify the suitability of MEDM for highly-unbalanced dataset}. The ground-truth category distribution can be calculated as $\mathbf{q}=\frac{1}{\sum_{i=1}^{12} l_i}[l_1,\cdots,l_{12}]$. Then, the entropy of $\mathbf{q}$ can be directly computed as $H(\mathbf{q})=2.3927$.

Table \ref{tb:visda50} compares various methods with the pretrained ResNet-50 architecture while Table \ref{tb:visda101} with the pretrained ResNet-101. Our method performs the best in the final mean accuracy among various methods. It surpasses the second best over 5\% in the final mean accuracy for the scenarios of both ResNet-50 and ResNet-101. With ResNet-101, MEDM achieves the record mean-accuracy of 82.4\%.

\subsection{ImageCLEF-DA}
ImageCLEF-DA is a publicly-available dataset for imageCLEF 2014 domain adaptation challenge. It has 12 common categories shared by the three public datasets: Caltech-256 (C), ImageNet ILSVRC 2012 (I), and Pascal VOC 2012
(P), which are also considered as three different domains. For 12 common categories, they are aeroplane, bike, bird, boat,
bottle, bus, car, dog, horse, monitor, motorbike, and people.  ImageCLEF-DA is a balanced dataset with 50 images in each category and 600 images in each domain. We consider all domain combinations and build 6 domain-adaptation tasks: I $\rightarrow$ P, P $\rightarrow$ I, I $\rightarrow$ C, C $\rightarrow$ I, C $\rightarrow$ P, and P $\rightarrow$ C.

Table \ref{table:clef} shows the classification accuracy results for various
methods on the ImageCLEF-DA dataset with the ResNet50 architecture. The result of the MEDM is obtained by only training 100 epoches, which, however, neatly outperforms the other deep adaptation methods among five adaptation tasks, I $\rightarrow$ P, P $\rightarrow$ I, C $\rightarrow$ I, C $\rightarrow$ P, P $\rightarrow$ C. The best average accuracy (88.9\%) is achieved by MEDM, which improves CDAN by about 1.2\%.

\subsection{Office-Home}
Office-Home \cite{CVPR17DHN}  is a typical dataset with a large number of classes (65 classes), which containing 15,500 images from four visually very different domains: \textbf{Ar}tistic images, \textbf{Cl}ip Art, \textbf{Pr}oduct images, and \textbf{R}eal-\textbf{w}orld images.  We consider all domain combination among these four domains, resulting 12 domain-adaptation tasks.

Table \ref{table:office-home} shows the classification accuracy results on the Office-Home dataset with the ResNet50 architecture. The result of the MEDM is obtained by only training 100 epoches and the best average accuracy (69.5\%) is achieved by MEDM, which improves MDD by about 1.4\%. This means that MEDM performs well for the datasets of large number of classes.

\subsection{Office-31}
Office-31 is a standard benchmark dataset for visual domain adaptation, which has 4652 images and 31
categories collected from three domains, Amazon (A), Webcam (W) and DSLR (D).
The Amazon (A) domain contains 2817 images downloaded from amazon.com. We consider all domain combination, resulting 6 domain-adaptation tasks.

For the transferring tasks over Office-31, we employ the same neural network architecture as ImageCLEF-DA. We compare the average classification accuracy of each method on 10 random experiments, and report the standard error of the classification accuracies by different experiments of the same transfer task. In all experiments, we train each model for 100 epochs and exceptions include D $\rightarrow$ A and W $\rightarrow$ A, where 200 epoches are employed.

We report the classification accuracy results on the Office-31 dataset as in Table \ref{table:office31}.
Office-31 has three domains of different sizes, which result into non-evenly distributed classes in each domain.

Among various domain-adaptation methods,  MEDM still performs the best for the mean accuracy. MEDM performs the best for three adaptation tasks, D $\rightarrow$ W, W $\rightarrow$ D, W $\rightarrow$ A, while MDD \cite{MDD} performs the best for the three remaining tasks.

\subsection{Ablation Study}
\subsubsection{Effect of $\lambda$ on Transferability}
\begin{table}
\caption{The effect of $\lambda$ on transferability for W $\rightarrow$ A}
\begin{center}
\label{my-label}
\begin{tabular}{l|c|c||l|c|c}
\toprule[1.5pt]
$\lambda+ \beta$ & $\mathcal{L}_e$    & Acc & $\lambda+\beta$  & $\mathcal{L}_e$  & Acc\\\hline
1.0 + 0 & 0.0 & {43.1} & 1.0 + 0.2 & 0.0 & {53.3}\\
0.8 + 0 & 0.0 & {45.9} & 0.8 + 0.2 & 0.0 & {56.5}\\
0.6 + 0 & 0.0 & {49.8} & 0.6 + 0.2 & 0.0 & {65.7}\\
0.4 + 0 & 0.1 & {54.0} & 0.4 + 0.2 & 0.1 & {74.3}\\
0.3 + 0 & 0.2 & {62.2} & 0.3 + 0.2 & 0.2 & {75.5}\\
0.2 + 0 & 0.3 & {63.2} & 0.2 + 0.2 & 0.3 & {74.5}\\
0.1 + 0 & 0.4 & {68.1} & 0.1 + 0.2 & 0.4 & {72.3}\\

\bottomrule[1.5pt]
\end{tabular}
\end{center}
\label{tb:gamma}
\end{table}

Entropy minimization in MEDM can be adjusted by varying $\lambda$.  As shown in Algorithm \ref{simProc}, MEDM encourages the use of small $\lambda$ whenever the target entropy approaches zero as the training iteration goes on. When $\beta=0$, the transferability is achieved by minimization of both the supervised loss on source domain and the entropy loss on target domain. With small $\lambda$ and keeping the (target) entropy small enough at the same time, it is expected that the end-to-end training of (\ref{total-loss}) ensures better transferability.

To investigate the choice of $\lambda$ on the final performance, we also show the accuracy of MEDM on the task $W\rightarrow A$ on the Office-31 dataset when $\lambda$ takes its value in \{0.2,0.3,0.4,0.6,0.8,1.0\} and $\beta \in \{0.0, 0.2\}$. With smallest possible value of $\lambda$ for $\mathcal{L}_e(\theta_i, \mathcal{T}) \rightarrow 0$, MEDM can achieve the best performance with a suitable choice of $\beta$ as shown in Table \ref{tb:gamma}. This means that the better transferability could be ensured with smaller possible value of $\lambda$ if $\mathcal{L}_e(\theta_i, \mathcal{T}) \rightarrow 0$ is satisfied at the end of training.

\subsubsection{Effect of $\beta$ on Diversity Maximization}
The superiority of MEDM in the VisDA challenge shows that it is very effective for highly-unbalanced target datasets, although the category diversity is expected to achieve its maximum value when the inferred categories are uniformly-distributed. We guess that it works well due to the collaboration in meeting both requirements, namely, the minimization of entropy and the maximization of category diversity, where the parameter $\beta$ (\ref{total-loss}) is used to balance two individual requirements.

To investigate the choice of $\beta$ on the final performance, we also show the accuracy of MEDM with ResNet-50 by fixing $\lambda=1.0$ and varying $\beta \in \{0.2,0.3,0.4,0.5\}$. As shown in Table \ref{tb:visda50-beta}, we observed that  $E_{\mathcal{T}}\{\mathcal{L}_d(\mathcal{T})\}$ after training 10 epoches is always less than the entropy of the ground-truth target category distribution $H(\mathbf{p})=2.3978$, which means that the maximization of $\mathcal{L}_d(\mathcal{T})$ under the constraint of entropy minimization does not necessarily produce the uniformly-distributed categories. When $\beta < 0.3$, it results into poorer performance as some categories (car/truck) simply fail to be identified. With the increase of $\beta$, the practical diversity also grows. When $\beta$ increases to 0.5, it also results into significantly worse performance compared to $\beta=0.4$. Essentially, individual entropy minimization may automatically tradeoff with diversity maximization \textit{if the values of $\lambda,\beta$ are properly validated by the use of DEV \cite{DEV}}.

\subsubsection{Further Inclusion of Domain Difference Loss}
We also investigate the possibility of a further inclusion of domain difference loss (e.g. MMD loss \cite{Long2015MMD} or adversarial loss \cite{Ganin2015domain}) in (\ref{total-loss}). However, experiments always show worse performance.

\section{Conclusion}
Entropy minimization has been shown to be a powerful tool for domain adaptation. However, entropy minimization is insufficient for the minimization of the target risk and trivial solutions are often observed. In this paper, we propose to employ diversity maximization for avoiding the trivial solutions. We show there exists a tradeoff for entropy minimization and diversity maximization towards the close-to-perfect domain adaptation. With the recently-proposed unsupervised model selection method, we show that the proposed MEDM outperforms state-of-the-art methods on several domain adaptation datasets, boosting a large margin especially on the largest VisDA dataset for cross-domain object classification.

\appendix[Trivial Solution Demonstration for Entropy-Minimization-Only Method]
\begin{figure}[htb] 
   \centering
   \includegraphics[width=0.45\textwidth]{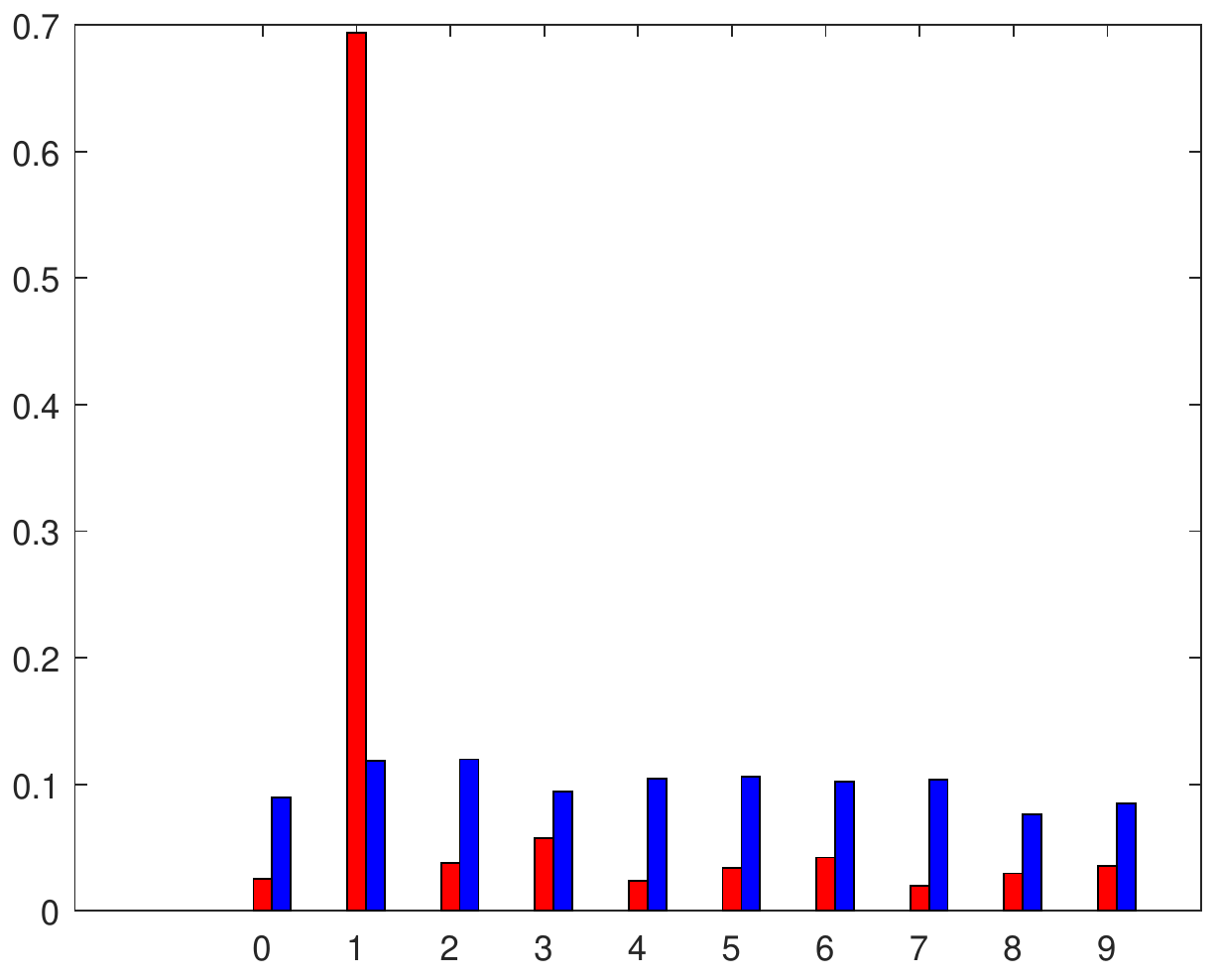} 
   \caption{Predicted category distribution (\ref{cat_prob}) for both MEDM (blue) and EMO (\ref{entropy-min}) (red) for a batch of target samples after a training iterations of 10000: EMO (Entropy-minimization-only) results into a trivial solution, where digit-1 dominates among others.}
   \label{fig:svhn}
\end{figure}

\begin{table}
\caption{Average accuracy (\%) for SVHN $\rightarrow$ MNIST}
\begin{center}
\begin{tabular}{l|c}
\toprule
Method& Acc \\
\midrule
RevGrad (\cite{Ganin2015domain}) & $73.9$ \\
ADDA (\cite{long2016RTN}) & $76.0$ \\
DTN (\cite{Taigman2017}) & $84.4$ \\
TRIPPLE (\cite{Saito2017domain}) & $86.2$ \\
 COREL (\cite{Sun2016ICCV}) & $90.2$ \\
 MECA (\cite{CPUA2018}) &$95.2$ \\
 \midrule
 EMO (\ref{entropy-min})  &$44.1$ \\
 MEDM(Ours) &$\textbf{98.7}\pm0.3$ \\
\bottomrule[1.5pt]
\end{tabular}
\label{table:svhn}
\end{center}
\end{table}
In this appendix, we consider the transfer task of SVHN $\rightarrow$ MNIST for demonstrating the trivial solutions of EMO.

The MNIST handwritten digits database has a training set of 60,000 examples, and a test set of 10,000 examples.
The digits have been size-normalized and centered in fixed-size images. SVHN is a real-world image dataset for machine learning and object recognition algorithms with minimal requirement on data preprocessing and formatting. It has 73257 digits for training, 26032 digits for testing. We focus on the task SVHN $\rightarrow$ MNIST in experiments.

We employed the CNN architecture used in \cite{Ganin2015domain}. The number of training iterations is set to $50000$ and the learning rate is set to 0.001. We run 10 experiments for computing average accuracy and its deviation.

Firstly, we show that EMO simply results into a trivial solution as indicated in Figure \ref{fig:svhn}, where digit-1 dominates among other categories for the model trained with entropy-minimization-only (\ref{entropy-min}) ($\lambda=1$). By inferring several target batches over the trained model, we observed that digit-1 always dominates for EMO. For MEDM, the predicted category distribution, however, is very close to the true uniform distribution.

Then, we compare our method with six methods in Table \ref{table:svhn} for unsupervised domain adaptation including state-of-the-art methods in visual domain adaptation: Reverse Gradient (RevGrad) \cite{Ganin2015domain}, Adversarial Discriminative Domain Adaptation (ADDA) \cite{Tzeng2017domain}, Domain Transfer Network (DTN) \cite{Taigman2017},TRIPPLE \cite{Saito2017domain}, CORrelation ALignment (CORAL) \cite{Sun2016ICCV}, Minimal-Entropy Correlation Alignment (MECA) \cite{Morerio2018MECA}. MEDM performs the best and it achieves the average accuracy of 98.7\%, which improves 3.5\% compared to MECA. As also shown in Table \ref{table:svhn}, EMO simply fails to work.


\end{document}